\documentclass[runningheads]{llncs}
\usepackage[T1]{fontenc}

\newcommand{\eqref}[1]{(\cref{#1})}
\usepackage{booktabs}
\usepackage{makecell}

\usepackage{lipsum}
\newcommand\blfootnote[1]{%
  \begingroup
  \renewcommand\thefootnote{}\footnote{#1}%
  \addtocounter{footnote}{-1}%
  \endgroup
}

\usepackage{graphicx}
\usepackage{tikz}
\usetikzlibrary{patterns}
\usepackage{subcaption}
\usepackage{hyperref}
\usepackage{leftidx}
\usepackage{amsmath} %
\usepackage{amsfonts}
\hypersetup{
    colorlinks=true,
    linkcolor=black,
    citecolor = black,
    filecolor=magenta,      
    urlcolor=blue
    }

\usepackage[capitalise]{cleveref}

\usepackage{graphicx}
\begin{document}
\title{Multi-Query Shortest-Path Problem in\\ Graphs of Convex Sets}

\author{
Savva Morozov \and
Tobia Marcucci \and
Alexandre Amice \and
Bernhard Paus Graesdal \and
Rohan Bosworth \and
Pablo A. Parrilo \and
Russ Tedrake
}
\authorrunning{S. Morozov et al.}

\institute{Massachusetts Institute of Technology}

\maketitle              %

\begin{abstract}
The Shortest-Path Problem in Graph of Convex Sets (SPP in GCS) is a recently developed optimization framework that blends discrete and continuous decision making.
Many relevant problems in robotics, such as collision-free motion planning, can be cast and solved as an SPP in GCS, yielding lower-cost solutions and faster runtimes than state-of-the-art algorithms.
In this paper, we are motivated by motion planning of robot arms that must operate swiftly in static environments.
We consider a multi-query extension of the SPP in GCS, where the goal is to efficiently precompute optimal paths between given sets of initial and target conditions.
Our solution consists of two stages.
Offline, we use semidefinite programming to compute a coarse lower bound on the problem's cost-to-go function.
Then, online, this lower bound is used to incrementally generate feasible paths by solving short-horizon convex programs.
For a robot arm with seven joints, our method designs higher quality trajectories up to two orders of magnitude faster than existing motion planners.

\blfootnote{This work was supported by
Amazon.com Services LLC, PO No. 2D-12585006; 
The AI Institute; 
Dexai Robotics;
National Science Foundation, DMS-2022448,
UC Berkeley, 00010918.
Corresponding author is Savva Morozov, \texttt{savva@mit.edu}.
}

\end{abstract}

\section{Introduction}
A Graph of Convex Sets (GCS)~\cite{marcucci2024graphs} is a graph where each vertex is paired with a convex set and an optimization variable inside this set, while each edge couples adjacent vertex variables through additional convex costs and constraints.
In the Shortest-Path Problem (SPP) in GCS~\cite{marcucci2024shortest}, we simultaneously seek a discrete path through this graph and optimize the continuous variables associated with the vertices along the path, while minimizing the cumulative edge costs.

Though the SPP in GCS is NP-hard~\cite[Section~9.2]{marcucci2024shortest}, effective solution methods have been proposed in~\cite{marcucci2024shortest,chia2024gcs}.
This technique has shown remarkable success in various robotics applications, such as optimal control~\cite{marcucci2024shortest}, planning through contact~\cite{graesdal2024towards}, and other robotics problems~\cite{philip2024mixed,kurtz2023temporal,cohn2023non}.
In real-world hardware deployment, it has been especially effective in collision-free motion planning~\cite{marcucci2023motion}, addressing the challenges of non-convex obstacle avoidance constraints.
\begin{figure}[!t]
\includegraphics[width=\textwidth]{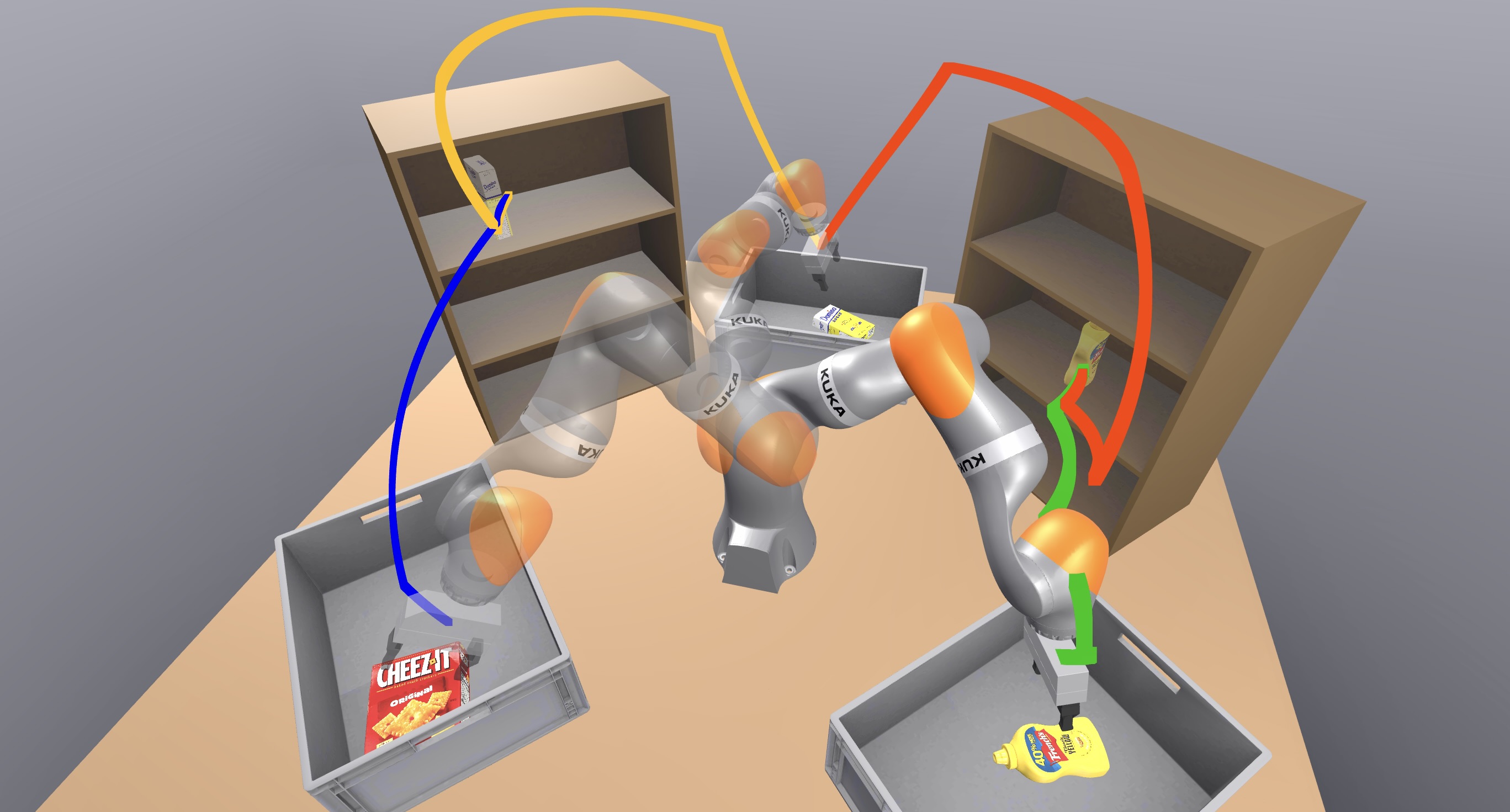}
\caption{
Robotic arm in a simulated environment, tasked with moving items between shelves and bins.
Shown are four queries for collision-free motion planning.
} \label{f-robot-arm-title}
\end{figure}
However, solving the SPP in GCS can sometimes be too slow for real-time applications on high-dimensional robotic systems.
Consider a 7-DoF KUKA iiwa robot arm repeatedly performing online motion planning in a static environment.
When the environment is simple, the GCS is small, and the shortest path queries can be solved quickly, in under 50ms~\cite{marcucci2023motion}.
However, when the environment is complex and the configuration space must be covered thoroughly, as in~\cref{f-robot-arm-title}, the GCS becomes large,
and the shortest path queries can take up to of 600ms.
This is not practical for high-productivity applications, such as robot arms in warehouses, where the company's income is nearly proportional to the operational speed.

In an effort to reduce solve times for online shortest path queries in GCS, we seek an efficient way of precomputing optimal paths between given sets of source and target conditions in the GCS.
We formulate this problem as a generalization of the SPP in GCS that is akin to the all-pairs generalization of the classical SPP.
Our solution contains two phases, illustrated in~\cref{f-gcs-example}.
Offline, we solve a semidefinite program that produces convex quadratic lower bounds to the cost-to-go function over the convex sets associated with GCS vertices.
Pictured in \cref{sf-g1} are the contour plots of these lower bounds at every vertex.
Then, online, we use a greedy multi-step lookahead policy with the cost-to-go lower bounds to determine the next vertex to visit.
Thus, as shown in \cref{sf-g2}, the path is obtained incrementally, one vertex at a time.
Though the quadratic cost-to-go lower bounds can be coarse, using the lookahead policy is equivalent to producing piecewise-quadratic lower bounds, which can be very expressive.
As a result, the obtained paths are nearly optimal in practice.
Convexity of the quadratic cost-to-go lower bounds allows us to evaluate the greedy policy by solving a set of small convex programs in parallel, which can be done quickly at runtime.
Applied to the complex scenario shown in \cref{f-robot-arm-title},
our method requires just 6s of offline computation to produce the cost-to-go lower bounds.
Subsequent online queries take 2-11ms, which is up to two orders of magnitude faster than solving the SPP in GCS from scratch.
\label{s-introduction}
\begin{figure}[t!]
    \centering
    \begin{subfigure}[t]{0.23\textwidth}
        \centering
        \includegraphics[width=\textwidth]{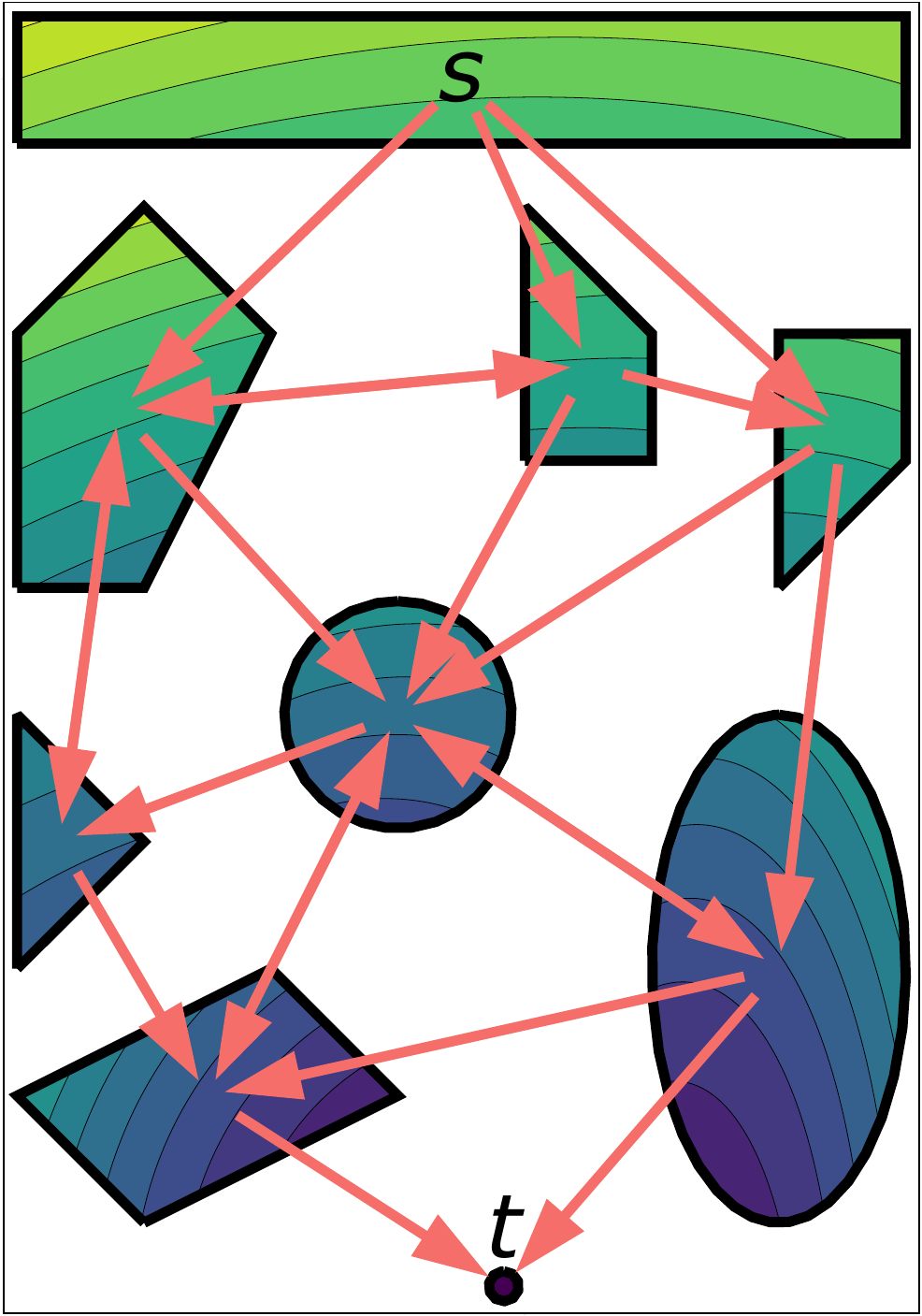}
        \caption{
            Offline: synthesize cost-to-go over the GCS.
            Contour plots are shown.
        }
        \label{sf-g1}
    \end{subfigure}
    \hspace{0.3cm}
    \begin{subfigure}[t]{0.71\textwidth}
        \begin{subfigure}[t]{0.325\textwidth}
            \centering
            \includegraphics[width=\textwidth]{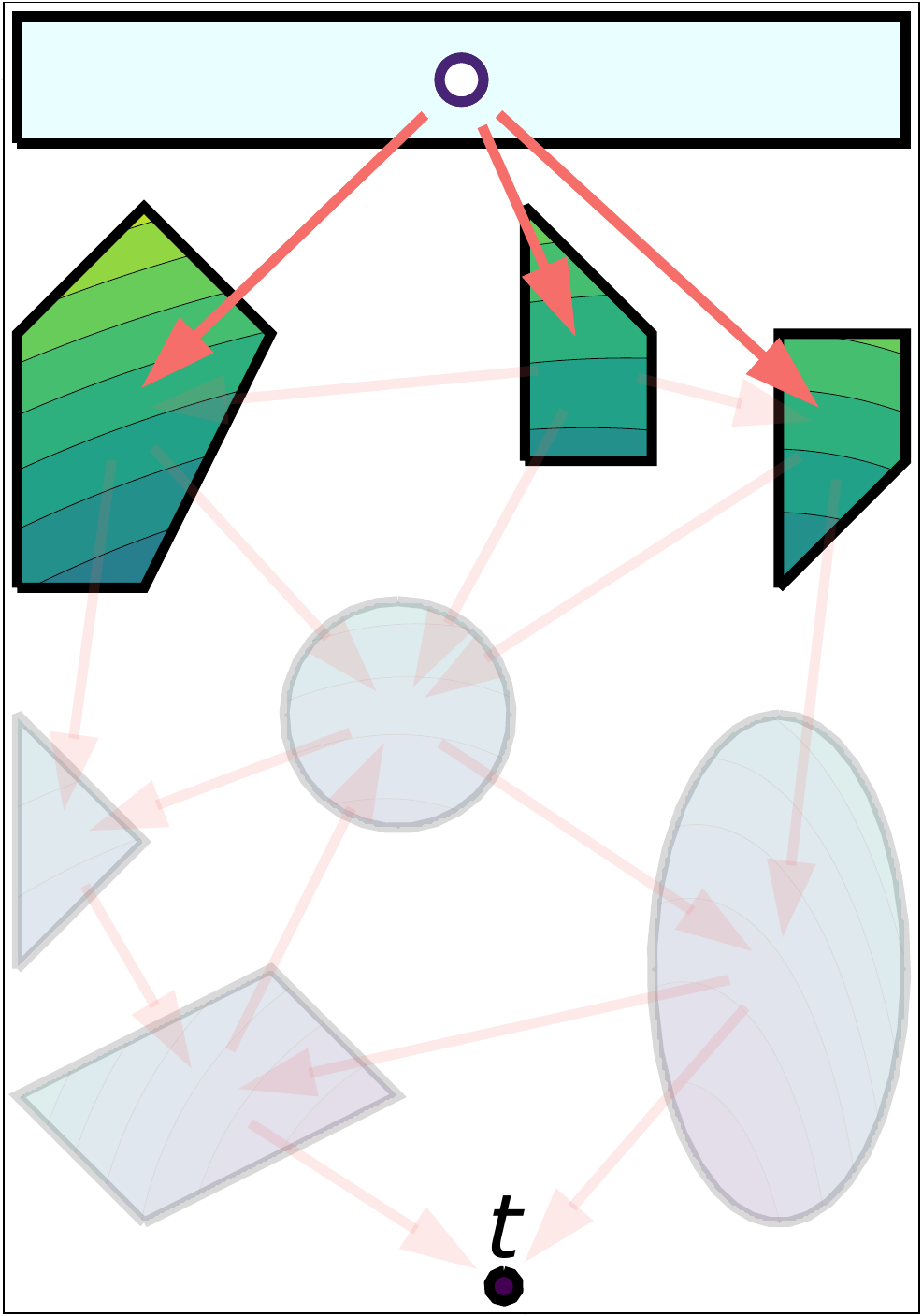}
        \end{subfigure}
        \begin{subfigure}[t]{0.325\textwidth}
            \centering\includegraphics[width=\textwidth]{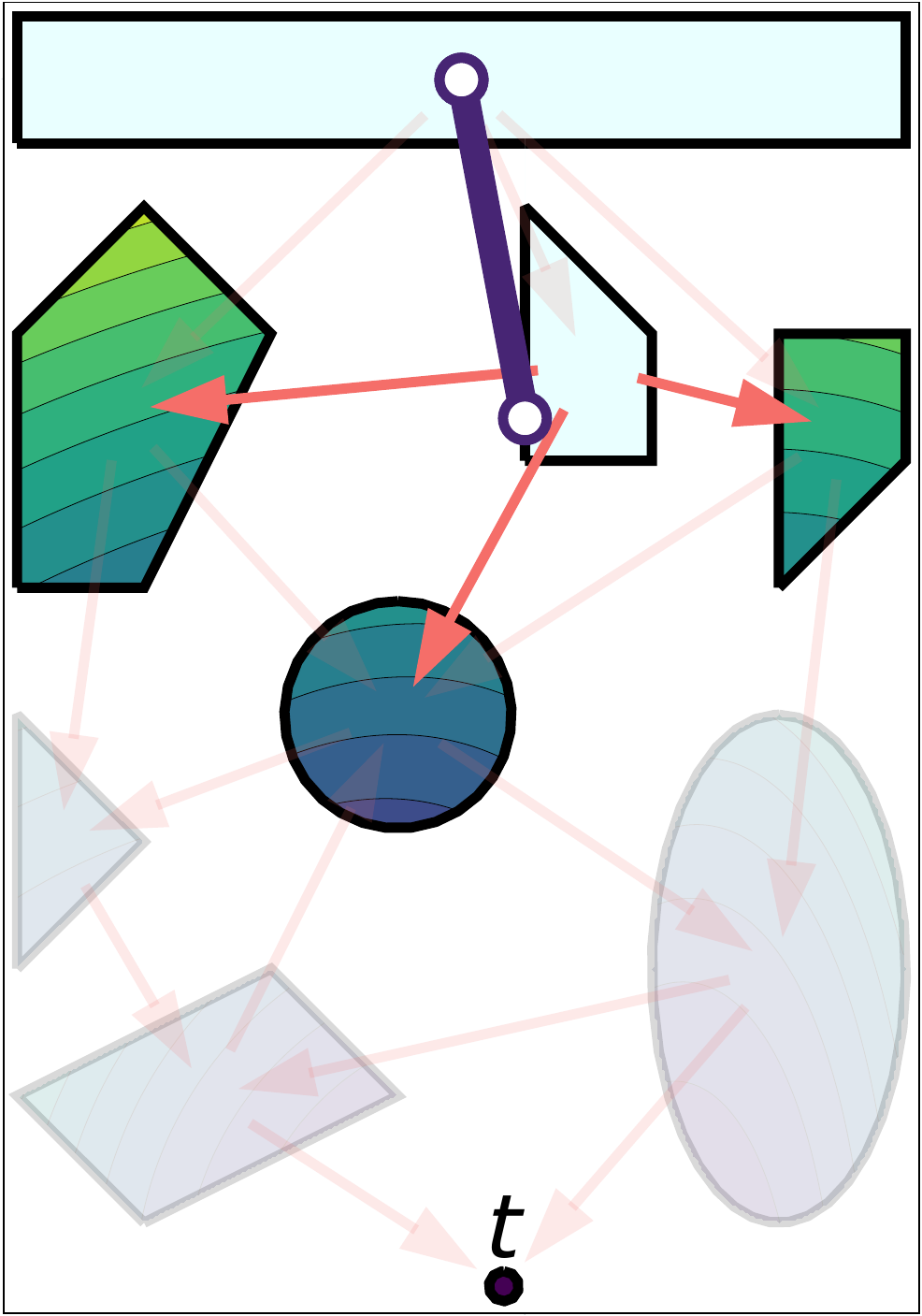}
        \end{subfigure}
        \begin{subfigure}[t]{0.325\textwidth}
            \centering\includegraphics[width=\textwidth]{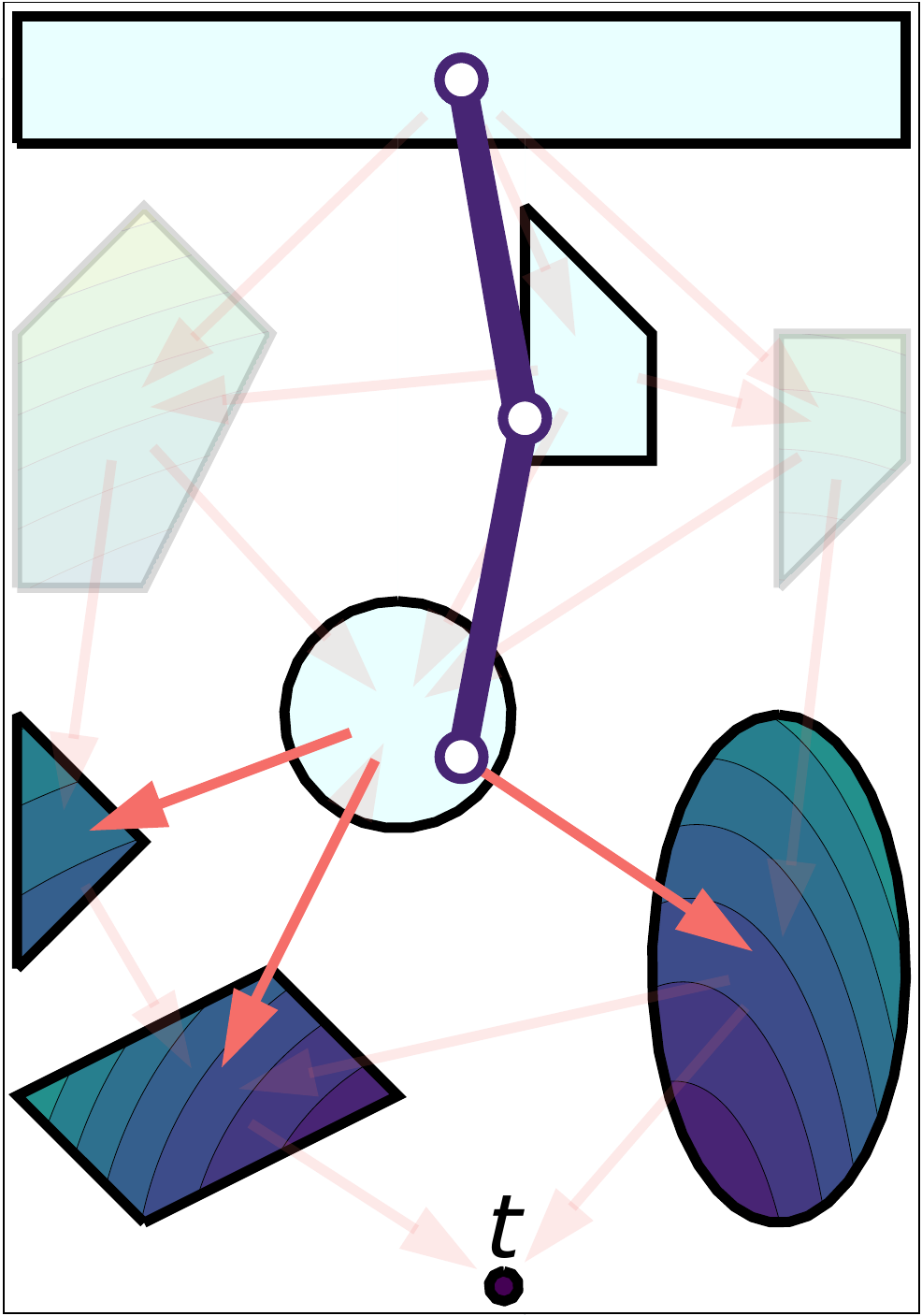}
        \end{subfigure}
        \caption{
            Online: at each iteration, we evaluate all $n$-step paths from the current vertex ($n\!=\!1$ shown) and greedily select the decision that minimizes the $n$-step lookahead cost-to-go.
            The first three iterations are shown, as the path is built incrementally.
            }
        \label{sf-g2}
    \end{subfigure}
    \caption{
    Illustration of our approach.
    The GCS instance is embedded in $\mathbb R^2$, with the source vertex at the top and the target vertex at the bottom. 
    The edges are shown as red arrows, and the edge length is the squared Euclidean distance.
    }
    \label{f-gcs-example}
\end{figure}

\subsection{Literature review}

\label{ss-related-work}
Graph search plays a central role in both modelling and solving a wide variety of planning problems in robotics. In this section we briefly connect our work to some notable examples in this literature.

A common approach to motion planning is to construct a graph where nodes correspond to collision-free configurations and edges correspond to collision-free motions. 
The most popular approaches based on this idea are the Rapidly exploring Random Trees (RRT)~\cite{lavalle1998rapidly}, Probabilistic Roadmap (PRM)~\cite{kavraki1996probabilistic}, and their many variants~\cite{kuffner2000rrt,bohlin2000path,jaillet2004prm,karaman2011sampling}. 
The GCS approach to motion planning is similar to the PRM one, but collision-free configurations are replaced with large collision-free sets~\cite{marcucci2023motion}.
GCS avoids two major drawbacks of planning with a PRM: the need to densely sample in high-dimensional spaces and post-process the motion plan to obtain a smooth trajectory.
However, generating these collision-free sets can be computationally challenging and expensive.
Furthermore, SPP in GCS queries can still be very expensive, motivating this current work.

The importance of the SPP has led to a breadth of literature on its solution, with Bellman's dynamic programming approach illustrating the central role of the cost-to-go function \cite{bellman1966dynamic,bertsekas2012dynamic}. 
Given the cost-to-go, a shortest path can be extracted using a simple greedy strategy: given a vertex $v$, the next vertex in the path is the one which minimizes the cost-to-go among all the neighbors of $v$.
This is captured by the famous Bellman's equation\cite{bellman1966dynamic}.

Multi-query SPP setting has also been thoroughly investigated. 
The All-Pairs Shortest Paths (APSP) is the problem of finding shortest paths between every pair of vertices in a discrete graph~\cite[Ch. 23]{cormen2022introduction}.
One method for solving this problem, from which we draw particular inspiration, first computes the cost-to-go function between every pair of vertices in the graph (also known as a \emph{distance oracle}). 
This cost-to-go is used to produce a successor along the shortest path between every pair of vertices, which is stored into the \textit{successor matrix}. 
The optimal paths are retrieved by sequentially querying this matrix.

Explicit solutions to the Bellman equation exist only in a handful of contexts. 
In the case of purely discrete graphs, a number of efficient methods exist~\cite{floyd1962algorithm,warshall1962theorem,johnson1977efficient}, where the cost-to-go function can be encoded using a simple matrix. 
Another notable example from control is Explicit Model Predictive Control (MPC) where the cost-to-go is a piecewise quadratic \cite{bemporad2002explicit}. 
However, even in these settings, storing the cost-to-go can be prohibitively expensive for large graphs, particularly in the APSP setting. 
In the purely discrete setting, the description of the APSP cost-to-go function grows quadratically in the size of the graph, while in the MPC setting it grows exponentially.

In most cases, solving the Bellman equation is known to be intractable. 
This has motivated a breadth of literature for computing approximations for the cost-to-go in various setting \cite{de2003linear,powell2007approximate,lasserre2008nonlinear,bertsekas2012dynamic,lewis2013reinforcement,wang2015approximate}.
Similarly, in this paper we seek a computationally tractable way to approximate the cost-to-go function to solve the APSP in GCS.
The generalization in the particular context of GCS is not straightforward and constitutes one of the contributions of this work.

\section{All-Pairs Shortest Paths in a Graph of Convex Sets}
\label{s-apsp-in-gcs}
We seek to efficiently precompute optimal solutions to the SPP in GCS between given sets of source and target conditions.
\cref{ss-apsp-in-graph} presents the classical APSP, which is the corresponding problem in an ordinary graph.
In \cref{ss-spp-in-gcs}, we describe the single-query SPP in GCS.
We then formulate the APSP in GCS in \cref{ss-apsp-in-gcs}, and outline our approximate solution method in \cref{ss-solution-overview}.

\subsection{All-Pairs Shortest Paths}
\label{ss-apsp-in-graph}

\paragraph{Graphs and paths.}
Let $G=(\mathcal V, \mathcal E)$ be a directed graph with vertex set $\mathcal V$ and edge set $\mathcal E$.
Given a source vertex $s$ and target vertex $t$, an $s\textsf{-}t$ {path} is a sequence of distinct vertices $p = (s\!=\!v_0, v_1, \ldots, v_K\!=\!t)$, where each consecutive pair of vertices is connected by an edge in $\mathcal E$ and no vertex is revisited. 
We define $\mathcal E_p = \{(v_0, v_1), \ldots, (v_{K-1}, v_K)\}$ as the set of edges traversed by the path $p$, and denote the set of all $s\textsf{-}t$ paths in $G$ as $\mathcal P_{s,t}$.

\paragraph{Shortest Path Problem (SPP).}

Let us associate with every edge $e\in\mathcal E$ a non-negative edge cost $c_e\in\mathbb R_+$.
A shortest path $p$ between the vertices $s$ and $t$ minimizes the sum of the edge costs along the path:
$$
\underset{p}{\min} \quad \sum_{e \in \mathcal E_p} c_e \quad  \text{s.t.} \quad p \in \mathcal P_{s,t}. 
$$
The optimal value of this program is called the \textit{cost-to-go} between $s$ and $t$, and is denoted by $J^*_{s,t}$.
\textit{The principle of optimality}~\cite{bellman1966dynamic} holds in this context, stating that every subpath of a shortest path is itself a shortest path.
This forms the foundation for many efficient solution algorithms to this problem.

\paragraph{All-Pairs Shortest Paths (APSP).} 
The APSP is the multi-query generalization of the SPP, where we seek a shortest path between all pairs of vertices in a graph.
Efficient solutions to the APSP leverage the principle of optimality.
Instead of computing the full path for each pair of vertices, it suffices to compute only the immediate successor along this path.
The full path can thus be attained incrementally, one vertex at the time.

This solution to the APSP can be implicitly encoded via the cost-to-go function $J^*_{v,t}$ for every pair of vertices $v$ and $t$, computed via dynamic programming~\cite[Ch. 23]{cormen2022introduction}~\cite{floyd1962algorithm,warshall1962theorem,johnson1977efficient}.
The successor is then computed by greedily picking a vertex that minimizes the one-step lookahead with respect to the cost-to-go:
\begin{subequations}
\label{e-successor-policy}
\begin{align}
\pi(v, t) \quad = \quad \underset{w}{\arg\min} \quad &  c_e + J^*_{w,t} \label{e-succp-a} \\
\text{s.t.} \quad & e = (v, w)\in\mathcal E. \label{e-succp-b}
\end{align}
\end{subequations}
The solution $\pi$ is a decision policy that, given the current and target vertices $v$ and $t$, selects the next vertex on the shortest $v\textsf{-}t$ path.
We refer to $\pi$ as the \textit{successor policy}.

\subsection{Shortest-Path Problem in a Graph of Convex Sets}
\label{ss-spp-in-gcs}

\paragraph{Graph of Convex Sets.}
A GCS is a directed graph $G = (\mathcal V, \mathcal E)$, where each vertex $v\in\mathcal V$ is paired with a bounded convex set $\mathcal X_v$ and a continuous variable $x_v\in\mathcal X_v$.
Each edge $e = (v,w)\in\mathcal E$ is then paired with a convex set $\mathcal X_e \subseteq \mathcal{X}_v \times \mathcal{X}_w$ and a convex non-negative edge length function $l_e: \mathcal{X}_e \rightarrow \mathbb{R}_+$, such that the adjacent vertex variables satisfy the constraint $(x_v,x_w)\in\mathcal X_e$, while minimizing the length $l_e(x_v,x_w)$~\cite{marcucci2024graphs}.

\paragraph{The Shortest Path Problem in a Graph of Convex Sets.}
The SPP in GCS between point $\bar x_s\in\mathcal X_s$ of vertex $s$ and $\bar x_t\in\mathcal X_t$ of vertex $t$ is defined as follows: 
\begin{subequations}\label{e-gcs-spp}
\begin{align}
\qquad\qquad \underset{p, \;\{x_v\}_{v \in p}}{\min} \quad & \sum_{e = (v,w)\in \mathcal E_p} l_{e}(x_v,\, x_w) &&\label{e-spp-a} \\
\text{s.t.} \quad & p \in\mathcal P_{s,t}, &&\label{e-spp-b}\\
& x_s = \bar x_s, \quad x_t = \bar x_t, &&\label{e-spp-c} \\
& x_{v}\in\mathcal X_{v}, &&\forall v\in p, \label{e-spp-d} \\
& (x_v,x_w) \in \mathcal X_{e}, &&\forall e=(v,w)\in \mathcal E_p. \label{e-spp-e}
\end{align}
\end{subequations}
Similar to the classical SPP, the SPP in GCS searches for an $s\textsf{-}t$ path $p=(v_0, v_1, \ldots, v_K)$ though a graph, which is a sequence of distinct vertices.
In addition to that, it also searches for a sequence of corresponding vertex variables ${(\bar x_s\!=\!x_{v_0}, x_{v_1}, \ldots, x_{v_K}\!=\!\bar x_t)}$, referred to as a \textit{trajectory}.
This trajectory satisfies the vertex and edge constraints~\eqref{e-spp-c},~\eqref{e-spp-d},~\eqref{e-spp-e}, while minimizing the edge costs~\eqref{e-spp-a}.
The optimal solution to~\eqref{e-gcs-spp} is thus a tuple (path and trajectory).
We denote the optimal value of~\eqref{e-gcs-spp} as $J^*_{s,t}(\bar x_s,\bar x_t)$ and refer to it as the \textit{cost-to-go} from point $\bar x_s$ of vertex $s$ to point $\bar x_t$ of vertex $t$.

Unlike the classical SPP, the SPP in GCS is NP-hard~\cite[Section~9.2]{marcucci2024graphs}, and thus unlikely to have a polynomial-time solution. 
However, it can be reformulated as a Mixed-Integer Convex Program (MICP) with a strong convex relaxation~\cite{marcucci2024shortest}: using a rounding strategy from~\cite{marcucci2023motion}, this relaxation often yields near-optimal solutions in practice.

For the classical SPP, the principle of optimality holds and the optimal policy is independent of past decisions, which simplifies the problem and enables many efficient solution algorithms. 
As demonstrated in the following example, these properties break down in the SPP in GCS.

\begin{example}
\label{ex-optimal-policy}
Consider the GCS in \cref{f-illustrative}, which is embedded in $\mathbb R^2$.
This GCS has four vertices $\mathcal V = \{s,v,w,t\}$, where the convex sets $\mathcal X_s,\mathcal X_v,\mathcal X_t$ are points, and the convex set $\mathcal X_w$ is a segment.
Every vertex is connected to every other vertex with an edge, and the edge lengths $l_e$ are the squared Euclidean distance (e.g., $l_{(v,w)} = ||x_v-x_w||_2^2$).

Due to the constraint that vertices cannot be revisited, the optimal policy is a function of the set of previously visited vertices.
This is demonstrated in \cref{sf-illustrative-1}, where we plot the optimal $s\textsf{-}t$ path in orange and the optimal $v\textsf{-}t$ path in blue.
The optimal decision at vertex $v$ depends on previously visited vertices: if $w$ was visited before, the optimal decision is to go to $t$ (orange), otherwise the optimal decision is to go to $w$ (blue).

\begin{figure}[t!]
    \centering
    \begin{subfigure}[t]{0.49\textwidth} 
        \centering        
        \includegraphics[width=\textwidth]{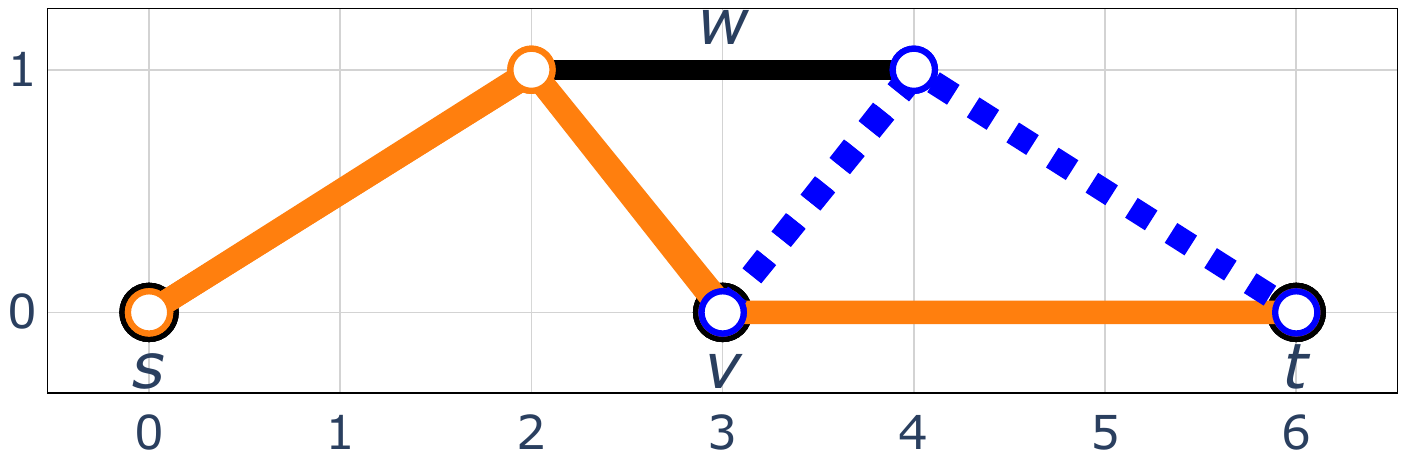}
        \caption{
        The optimal policy at vertex $v$ depends on  previous decisions.
        If $w$ has been visited already, the optimal decision is to go to $t$ (orange), otherwise it is to go to $w$ (blue).
        }
        \label{sf-illustrative-1}
    \end{subfigure}\hfill
    \begin{subfigure}[t]{0.49\textwidth}
        \centering
        \includegraphics[width=\textwidth]{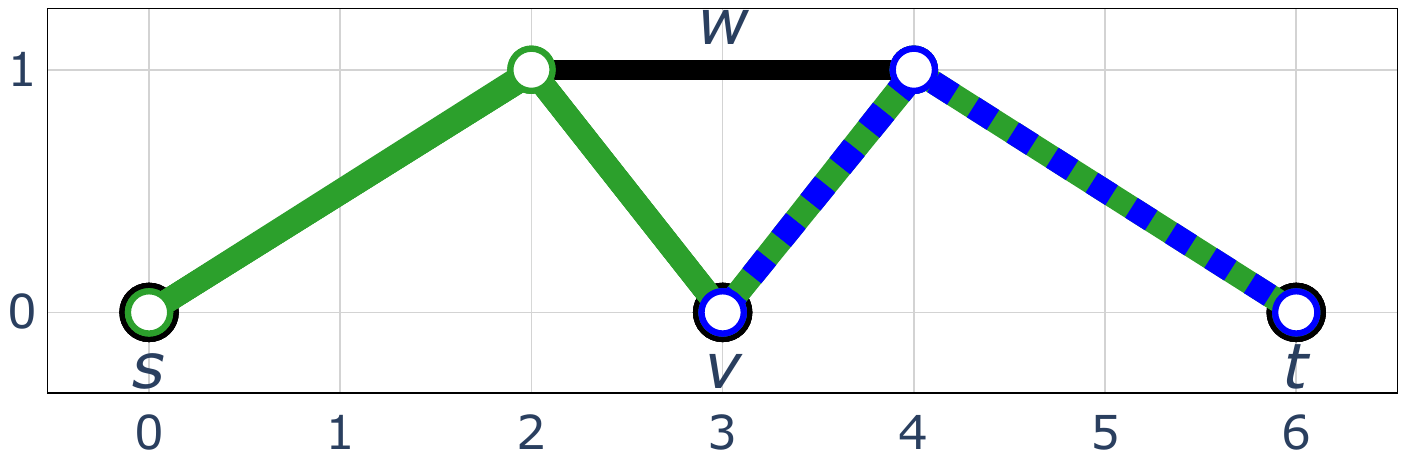}
        \caption{
        If we allow vertex revisits, the optimal policy is independent of past decisions.
        Shown are the optimal $s\textsf{-}t$  (green) and $v\textsf{-}t$ (blue) solutions to the relaxed problem.
        }
        \label{sf-illustrative-2}
    \end{subfigure}
    \caption{
    The two-dimensional GCS from \cref{ex-optimal-policy}.
    The convex sets paired with $s,v,t$ are points and the one paired with $w$ is a segment. 
    The GCS is fully connected, and the edge lengths are the squared Euclidean distance. 
    }
    \label{f-illustrative}
\end{figure}

Observe also that the principle of optimality does not hold for this problem: the $v\textsf{-}t$ subpath of the optimal $s\textsf{-}t$ path (orange) is not the optimal $v\textsf{-}t$ path (blue).
We cannot substitute the optimal $v\textsf{-}t$ path (blue) in place of the original $v\textsf{-}t$ subpath, since the resulting vertex sequence $(s,w,v,w,t)$ (\cref{sf-illustrative-2}, green) visits vertex $w$ twice, and is therefore not a path.

The constraint that vertices cannot be revisited is a key challenge of the SPP in GCS.
This is unlike the classical SPP with non-negative edge lengths, where this constraint does not increase problem complexity.
It can be shown that if we allow vertex revisits, then the principle of optimality holds, and the optimal decision policy is independent of past decisions.
This is illustrated in \cref{sf-illustrative-2}, where the optimal $s\textsf{-}t$ and $v\textsf{-}t$ solutions to the relaxed problem are shown in green and blue respectively.

\end{example}

\subsection{All-Pairs Shortest Paths in a Graph of Convex Sets}
\label{ss-apsp-in-gcs}

The APSP in GCS extends the classical APSP in a natural way.
We are given a set of source vertices $\mathcal S\subset \mathcal V$ and a set of target vertices $\mathcal T\subset \mathcal V$.
The goal is to solve the SPP in GCS between every pair of source and target points $\bar x_s \in \mathcal X_s$ and $\bar x_t \in \mathcal X_t$, and every pair of source and target vertices $s \in \mathcal S$ and $t \in \mathcal T$.
Since the SPP in GCS is NP-hard, the APSP in GCS is at least NP-hard as well. 

\subsection{Method outline}
\label{ss-solution-overview}

Our approach generalizes the solution to the classical APSP outlined in \cref{ss-apsp-in-graph}.
We proceed in two phases.
Offline, we compute a coarse quadratic lower bound on the cost-to-go between relevant pairs of GCS vertices.
Then online, we extend the greedy policy~\eqref{e-successor-policy} to the GCS setting.
At runtime, we rollout this policy to obtain the solution path incrementally, one vertex at a time.

Unlike the classical APSP, a greedy policy with the cost-to-go $J^*_{s,t}$ is not an optimal policy for the APSP in GCS.
This is because the optimal policy for paths in GCS depends on previously visited vertices, which is not captured by the cost-to-go  $J^*_{s,t}(x_s,x_t)$.
Thus, our approach is bound to yield approximate solutions, further limited by the coarseness of quadratic cost-to-go lower bounds.

To incorporate the challenging ``no-vertex-revisit constraint'' into the cost-to-go function, we relax this constraint by introducing penalties for vertex revisits.
These penalties are applied to the edge lengths, producing a biased cost-to-go lower bound that discourages revisits.
When rolling out a greedy policy online, we also explicitly prohibit vertex revisits.
To mitigate the coarseness of quadratic cost-to-go lower bounds and better approximate the optimal policy, we employ a multi-step lookahead generalization of the greedy policy~\eqref{e-successor-policy}, optimizing over $n$-step decision sequences at each iteration.

\section{Offline phase: synthesis of cost-to-go lower bounds}
\label{s-cost-to-go-function-synthesis}

In \cref{ss-cost-to-go-synthesis}, we present the optimization problem that produces cost-to-go lower bounds for the APSP in GCS.
This program is infinite-dimensional, so in \cref{ss-sdp} we present a tractable numerical approximation for it.

For clarity of presentation, we make some simplifying assumptions.
First, we assume that we have just one source vertex and one target vertex, i.e., $\mathcal S=\{s\}$ and $\mathcal T=\{t\}$. 
Second, we assume that the set $\mathcal X_t$ corresponding to the target vertex $t$ is a singleton: $\mathcal X_t = \{ x_t\}$.
Since the target vertex $t$ and point $x_t$ are fixed, we also simplify the notation and refer to $J_{v,t}^*(x_v, x_t)$ as $J^*_v(x_v)$.
The extensions of our method when these assumptions do not hold are straightforward and discussed in \cref{a-further-generalization}.

\subsection{Cost-to-go lower bounds via infinite-dimensional LP}
\label{ss-cost-to-go-synthesis}

The cost-to-go lower bounds are synthesized with the following optimization problem:
\begin{subequations}
\label{e-path-cost-to-go-synthesis}
\begin{align}
\max_{\{J_v,  h_v\}_{v \in \mathcal V}}  \quad &  
\int_{\mathcal X_s} J_s(x) d\phi_s(x) && \label{e-p-objective} \\
\text{s.t.} \quad & J_v: \mathcal X_v\rightarrow\mathbb R, &&\forall v\in \mathcal V, \label{e-p-value-def} \\
& h_w \geq 0  &&\forall w\in\mathcal V,\label{e-p-penalty-def}\\
&  J_v(x_v) \leq l_e(x_v, x_w) + h_w + J_w(x_w),  && \forall e=(v,w)\in\mathcal E,\! \label{e-p-lower-bound}\\
&&& \forall (x_v,x_w)\in\mathcal X_e,\! \nonumber\\
& J_t(x_t) = - \sum_{w\in\mathcal V} h_w. \label{e-p-target-value}
\end{align}
\end{subequations}
We now give a detailed line-by-line explanation of this program, and prove the validity of the lower bounds it produces in \cref{theorem-prog-lower-bound} below.

In constraint~\eqref{e-p-value-def}, we associate with every vertex $v\in\mathcal V$ a (possibly non-convex) function~$J_v$ defined over the set $\mathcal X_v$.
These functions serve as the lower bounds on the cost-to-go $J_v^*$, as will be shown later.
We emphasize that we are searching over the space of functions $J_v$, not over the individual points $x_v$.

In the objective function \eqref{e-p-objective}, $\phi_s$ is a probability distribution of anticipated source conditions over the set $\mathcal X_s$.
Thus, the integral in \eqref{e-p-objective} maximizes the weighted average of $J_s$ over the source set $\mathcal X_s$, effectively ``pushing up'' on the cost-to-go lower bound at the source vertex.

In \eqref{e-p-penalty-def}, we introduce a non-negative penalty $h_w$ for every vertex $w\in\mathcal V$.
This penalty is meant to discourage revisits to vertex $w$, which is a way to relax the constraint that a path must not visit any vertex more than once.

To implement the penalty $h_w$, we increment the edge length $l_e$ for every edge $e\in\mathcal E$ that enters vertex $w$.
This is formalized in \eqref{e-p-lower-bound}, which states that for every edge $e = (v,w)$ and a feasible transition $(x_v,x_w)\in\mathcal X_e$, the value $J_v(x_v)$ is a lower bound on the sum of the penalty-incremented edge length $l_e(x_v,x_w) + h_w$ and the subsequent cost-to-go lower bound $J_w(x_w)$.
As written, the non-negative penalty $h_w$ increases the cost of the edges leading into vertex $w$, thereby discouraging visits to $w$.
However, since our goal is to only discourage vertex revisits, we need to waive the penalty $h_w$ once.
This is achieved by setting the cost-to-go lower bound $J_t(x_t)$ to $-\sum_{w\in\mathcal V} h_w$ in constraint \eqref{e-p-target-value}.
Upon reaching the target vertex, we subtract the sum of all vertex penalties from the cost-to-go lower bound, effectively waiving the penalties once per vertex.
We now show that these constraints produce lower bounds on the cost-to-go function.

\begin{lemma}
\label{theorem-prog-lower-bound}
Let $J_v$ and $h_v$ for $v\in \mathcal V$ be a feasible solution of problem~\eqref{e-path-cost-to-go-synthesis}.
Then
$$
J_v(x_v)\leq J_v^*(x_v) \text{ for all } v\in\mathcal V.
$$
\end{lemma}
\begin{proof}
Consider the optimal solution to program~\eqref{e-path-cost-to-go-synthesis}, and let $v$ be some vertex.
Let $p$ be an optimal path from a point $x_v\in\mathcal X_v$ to the target point $x_t$.
Since $p$ is a path, it contains no repeated vertices.
Adding the constraint~\eqref{e-p-lower-bound} along the edges $\mathcal E_p$ of this optimal path, we have:
\begin{align}
\label{e-lower-bound-penalties}
J_v(x_v) \leq \sum_{e=(u,w)\in\mathcal E_p} l_e(x_u,x_w)  \sum_{w\in p} h_w + J_t(x_t),
\end{align}
where $x_u,x_w$ are the vertex variables of the optimal trajectory corresponding to $p$.
Constraint~\eqref{e-p-target-value} states that $J_t(x_t)=- \sum_{w\in\mathcal V} h_w$, while the sum of the edge lengths $l_e(x_u,x_w)$ along the optimal path $p$ is by definition the cost-to-go $J_v^*(x_v)$.
Substituting and rearranging terms, we obtain:
\begin{align}
\label{e-lower-bound-penalties-2}
J_v(x_v)+ \sum_{w\notin p} h_w \;\leq\;  J_v^*(x_v).
\end{align}
Since the penalties $h_w$ are non-negative by~\eqref{e-p-penalty-def}, the conclusion follows.~\qed
\end{proof}
By maximizing the weighted average of $J_s$ in the objective function~\eqref{e-p-objective}, the program~\eqref{e-path-cost-to-go-synthesis} seeks the best possible lower bound $J_s$ on the cost-to-go $J_s^*$, up to the relaxation gap introduced by the vertex penalties.
This gap is clear from~\eqref{e-lower-bound-penalties-2}: for $x_s\in\mathcal X_s$, the sum of the off-the-optimal-path penalty terms $\sum_{w\notin p} h_w$ need not to be zero,
so $J_s(x_s)$ need not be a tight lower bound on $J_s^*(x_s)$.
In other words, recall that, upon reaching the target, we waive the penalties $h_w$ for every vertex $w\in\mathcal V$.
As a result, we do not just waive the first-time penalties on vertices along the optimal path $p$, we also waive the off-the-path penalties $\sum_{w\notin p} h_w$, which were never accrued in the first place.
Waiving these off-the-path penalties introduces the gap between $J_s$ and $J_s^*$.
\paragraph{Example 1, continued.}
Consider the solution to program~\eqref{e-path-cost-to-go-synthesis} for the GCS instance in~\cref{f-illustrative}.
Setting the revisit penalty $h_w = 0$ results in $J_s=14$, which is the cost of the vertex sequence that visits $w$ twice (green in \cref{sf-illustrative-2}).
By jointly optimizing over the penalties and the cost-to-go lower bounds, program~\eqref{e-path-cost-to-go-synthesis} selects the penalty $h_w=2$.
Revisiting vertex $w$ is no longer advantageous, 
and the cost of the shortest $s\textsf{-}t$ path (orange in~\cref{sf-illustrative-1}) is achieved: $J_s = J_s^* = 16$.

\vspace{\baselineskip}

We note that program \eqref{e-path-cost-to-go-synthesis} naturally generalizes the cost-to-go synthesis LP for the classical SPP~\cite{de2003linear}. 
When each convex set $\mathcal X_v$ is a singleton, the problem reduces to the classical SPP, where functions $J_v$ are defined at single points and represented by a single decision variable. 
Setting vertex penalties $h_w=0$ recovers the standard cost-to-go synthesis LP for the classical SPP:
\begin{equation}
\begin{aligned}
\label{discrete-cost-to-go-search}
\max_{\{J_v\}_{v\in\mathcal V}}  \quad &  J_s \\
\text{s.t.} \quad &  J_v \leq l_e + J_w,   \qquad&& \forall e=(v,w)\in\mathcal E,\\
& J_t = 0.
\end{aligned}
\end{equation}
Compared to the purely discrete setting of~\eqref{discrete-cost-to-go-search}, optimization program~\eqref{e-path-cost-to-go-synthesis} is also an LP; however, it searches over the space of functions and is therefore infinite-dimensional. 
Next, we develop a tractable finite-dimensional approximation to~\eqref{e-path-cost-to-go-synthesis} that is conducive to numerical methods.

\subsection{Numerical approximation via semidefinite programming}
\label{ss-sdp}

We now produce an approximate solution to the cost-to-go synthesis program~\eqref{e-path-cost-to-go-synthesis}.
We restrict each function $J_{v}$ to be convex quadratic, which allows us to cast~\eqref{e-path-cost-to-go-synthesis} as a tractable Semidefinite Program (SDP). 
SDPs are mathematical programs where the objective function is linear and the constraints are either linear or linear matrix inequalities (LMIs).
To help with the presentation, we first state without proof three well-known facts.
 
\begin{lemma}[{e.g., \cite[App.~A.1]{blekherman2012semidefinite}}] \label{lemma-quadratic}
A quadratic function $f : \mathbb R^n \rightarrow \mathbb R$ is non-negative if and only if it is representable as a Positive-Semidefinite (PSD) quadratic form:
\begin{align} 
f(x) = \begin{bmatrix}1\\x\end{bmatrix}^\top\! Q \begin{bmatrix}1\\x\end{bmatrix} \; \text{ for some } \; Q\succeq 0.\nonumber
\end{align}
\end{lemma}
\begin{lemma}[{\cite[Section 3.2.4]{blekherman2012semidefinite}}] \label{lemma-non-negative-on-set}
Let $\mathcal{X} = \{x \in \mathbb{R}^n\;|\; g_i(x)\geq 0, \; \forall i=1,\dots,m\}$.
The function $f : \mathbb R^n \rightarrow \mathbb R$ is non-negative on the set $\mathcal X$ if there exists $\lambda\in\mathbb R^m_+,$ such that $f(x) - \sum_{i=0}^m\lambda_{i}g_{i}(x)$ is non-negative for every $x\in\mathbb R^n$.
\end{lemma}
\begin{corollary}
\label{lemma-verify-with-quadratic}
Suppose that in \cref{lemma-non-negative-on-set}, the function $f$ is quadratic, and all $g_i$ functions are affine or convex quadratic.
Then we can apply \cref{lemma-quadratic} to verify \cref{lemma-non-negative-on-set} via an LMI, i.e., we can verify if $f$ is non-negative over $\mathcal X$ by searching for a PSD matrix in an affine subspace.
\end{corollary}

Using these facts, we proceed to cast program~\eqref{e-path-cost-to-go-synthesis} as an SDP.
\paragraph{Defining cost-to-go lower bounds in~\eqref{e-p-value-def}.}
We restrict lower bounds $J_v$ per vertex $v\in\mathcal V$ to be convex quadratic functions.
By \cref{lemma-quadratic}, searching for such functions is equivalent to searching for appropriate PSD matrices $Q_v$.
The decision variables are thus the coefficients of the quadratic polynomials.
As a result, we produce coarse quadratic lower bounds on the optimal~$J^*_v$; this coarseness will be mitigated via the multi-step lookahead policies.

Constraint~\eqref{e-p-penalty-def} is already linear, and constraint~\eqref{e-p-target-value} is linear in the coefficients of the quadratic polynomial $J_t$ and the decision variables $h_w$.
These constraints are thus already suitable for the SDP. 

\paragraph{Enforcing the lower-bound constraint~\eqref{e-p-lower-bound}.}
To apply \cref{lemma-verify-with-quadratic} to enforce this constraint, we impose additional restrictions.
First, we restrict vertex and edge sets $\mathcal X_v$ and $\mathcal X_e$ to be intersections of ellipsoids and polyhedra.
We also restrict edge lengths $l_e$ to be quadratic, ensuring that the expression in~\eqref{e-p-lower-bound} is quadratic.
For non-quadratic $l_e$, such as the Euclidean distance, we use a quadratic approximation instead.
Applying \cref{lemma-verify-with-quadratic}, we verify constraint~\eqref{e-p-lower-bound} with an LMI.

\paragraph{The objective function~\eqref{e-p-objective}.} 
Since $J_s$ is a quadratic polynomial, the integral in~\eqref{e-p-objective} is linear in the coefficients of $J_s$, which are the decision variables of the program. 
Therefore, the objective function~\eqref{e-p-objective} is linear in the decision variables, as required for the SDP.

\vspace{\baselineskip}

Empirically, we found quadratic lower bounds to be a good balance between computational complexity and expressive power.
Note that higher-degree polynomial lower bounds $J_v$ can be synthesized via the Sums-of-Squares (SOS) hierarchy~\cite{parrilo2000structured,parrilo2003semidefinite,lasserre2001global}.
However, in practice, the resulting programs tend to be prohibitively expensive.
On the other hand, restricting $J_{v}$ to be affine yields a program that almost exactly matches the dual to the convex relaxation of the SPP in GCS, discussed in~\cite[App.~B]{marcucci2024shortest}.
In other words, solving the SPP in GCS already gives a coarse affine cost-to-go lower bound that can be used to solve the APSP in GCS.
In \cref{ss-ablation-example}, we show that empirically, these affine lower bounds have significantly less expressive power than the quadratic lower bounds.

\section{Online phase: greedy multi-step lookahead policy}
\label{s-lookahead-policies}
We now generalize the greedy successor policy \eqref{e-successor-policy} from the classical APSP to the GCS setting.
Suppose that at runtime, we are given a source vertex $v_0 \in \mathcal S$ and a source point $x_0 \in \mathcal X_{v_0}$.
At iteration $k$ of the policy rollout, let $(v_k,x_k)$ be the current vertex and vertex variable, and let $p_k = (v_0, v_1, \ldots, v_{k-1})$ be the path so far.
The successor policy $\pi(v_k, x_k, p_k) = (v_{k+1}, x_{k+1})$, which we will define shortly, produces the next vertex $v_{k+1}$ and the corresponding vertex variable $x_{k+1}$.
We then advance to the next iteration.
The rollout terminates when we reach the target vertex $t$, where we must select the target point $x_t$.
Upon termination, we extract the vertex path $p = (v_0, v_1, \ldots, v_t)$ and re-optimize for the continuous vertex variables $(x_0, x_1, \ldots, x_t)$, so as to produce a trajectory that is optimal within this path. 

At each iteration of the policy rollout, we solve a greedy lookahead optimization problem with the coarse quadratic lower bounds obtained in \cref{ss-sdp}.
For simplicity, here we present just the 1-step lookahead program:
\begin{subequations}
\label{e-policy}
\begin{align}
\pi(v_k ,x_k, p_k) \quad=\quad \underset{(w, x_w)}{\arg\min}& \quad  l_e(x_k, x_w) + J_{w}(x_w) \label{e-policy-objective} \\
\text{s.t.} & \quad e=(v_k,w)\in\mathcal E, \quad w\notin p_k, \label{e-policy-discrete-con} \\
& \quad (x_k,x_w)\in\mathcal X_e. \label{e-policy-continuous-con}
\end{align}
\end{subequations}
Note that we do not use the penalty-incremented edge cost $l_e(x_k,x_w) + h_w$ in \eqref{e-policy-objective}, since the penalty $h_w$ is waived the first time that $w$ is visited. 
Vertex revisits are then also explicitly prohibited in~\eqref{e-policy-discrete-con}.

In a multi-step lookahead formulation, we instead solve for an $n$-step optimal decision sequence, take just the first step, and repeat at next iteration.
The multi-step lookahead is key for mitigating the coarseness of the quadratic lower bounds.
This is because an $n$-step lookahead from vertex $v$ effectively produces a piecewise-quadratic lower bound on the cost-to-go $J_v^*$ over $\mathcal X_v$, which has significantly more expressive power.
While these lower bounds can still be loose in theory, the multi-step lookahead enables effective decision-making in practice.

Convexity of $J_{w}$ is crucial, as it allows us to  solve the program~\eqref{e-policy} efficiently at run-time.
To find the minimizer to~\eqref{e-policy}, we solve multiple convex programs in parallel, one for every $n$-step lookahead sequence.

Finally, we note that the lookahead program~\eqref{e-policy} is not guaranteed to be  recursively feasible.
If we end up in a vertex where~\eqref{e-policy} has no solution, we backtrack to a previous vertex that has a different feasible outgoing edge, and retry from there.
Generally, our planner is sound but not complete: it is not guaranteed to produce a solution, but every solution it produces is feasible.

\section{Experimental evaluation}
\label{s-experiments}
We evaluate our approach through multiple numerical experiments.
\cref{ss-simple-intuitive-example} presents a simple two-dimensional problem that provides visual intuition to our method.
In \cref{ss-robot-arm-example}, we apply our approach to a complex high-dimensional scenario, the robot arm in \cref{f-robot-arm-title}.
Finally, \cref{ss-ablation-example} shows that our approach scales well to large graphs. 
We also discuss how the coarseness of the cost-to-go lower bounds and the multi-step lookahead horizon impact the performance.

All of the experiments are run on a desktop computer with a 4.5Ghz 16-core AMD Ryzen 9 processor and 64GB 4800MHz DDR5 memory.
We use Mosek 10.2.1~\cite{mosek} to solve all the convex programs in this section.

\subsection{Two-dimensional example}
\label{ss-simple-intuitive-example}

We first consider a two-dimensional GCS problem in~\cref{f-gcs-example}.
We have a graph $G$ with $|\mathcal V| = 9$ vertices, $|\mathcal E| = 25$ edges, including multiple cycles.
The geometry of the convex sets $\mathcal X_v$ can be deduced from~\cref{sf-g1}; no edge constraints $\mathcal X_e$ are used.
The edge costs $l_e(x_v,x_w) = \|x_v-x_w\|_2^2$ are the squared Euclidean distance.
The source vertex $s$ is a box, and the target vertex $t$ is a singleton.

We compute the convex quadratic lower bounds on the cost-to-go function at every vertex and visualize their contour plots in~\cref{sf-g1}.
In~\cref{sf-g2}, we depict the first three iterations of the 1-step lookahead rollout of the successor policy~\eqref{e-policy}.
At each iteration, we expand the neighbours of the current vertex and greedily select the next vertex $w$ and the vertex point $x_w$ that minimize the objective~\eqref{e-policy-objective}.
The rollout proceeds until the target vertex $t$ is reached.

\begin{figure}[t!]
    \centering
    \includegraphics[width=\textwidth]{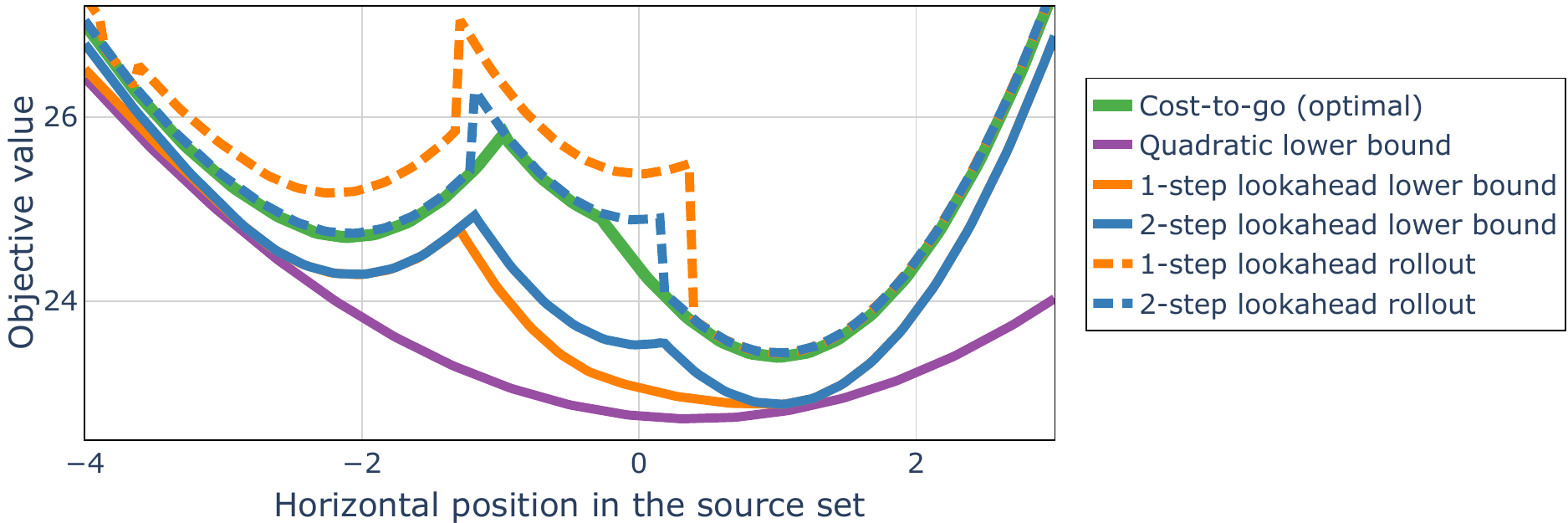}
    \caption{
    Comparison of lower and upper bounds on the cost-to-go over a horizontal slice of the source set $\mathcal X_s$ from \cref{sf-g1}.
    The cost-to-go function $J_s^*$ (green) is piecewise-quadratic.
    Convex quadratic lower bound $J_s$ (purple) is naturally a poor lower-bound.
    Multi-step lookaheads (solid orange, blue) produce tighter piecewise-quadratic lower bounds.
    Upper bounds on the cost-to-go are obtained by rolling out the multi-step lookahead policy (dashed orange, blue), which produces near-optimal solutions.
    }
    \label{f-lookahead-comparison}
\end{figure}

We evaluate the quality of the cost-to-go lower bounds and the resulting solutions in \cref{f-lookahead-comparison}.
The optimal shortest path cost-to-go function $J_{s}^*$ (green) is piecewise-quadratic.
Naturally, the convex quadratic lower bound $J_s$ (purple) is a poor lower bound to $J_{s}^*$.
The quality of the lower bound is greatly improved via multi-step lookaheads (solid lines, orange for 1-step, blue for 2-step).
A horizon-$n$ lookahead produces a piecewise-quadratic lower bound to $J_{s}^*$, with up to as many quadratic pieces as there are different $n$-step paths from the source vertex~$s$.
Though neither 1-step nor 2-step lookahead lower bounds are tight, they are sufficient for near-optimal decision making.
The costs of the rollouts of the successor policy are plotted as dashed lines; 2-step lookahead rollouts (blue) attain optimal solutions nearly always.

\subsection{Collision-free motion planning for a robot arm}
\label{ss-robot-arm-example}

We now demonstrate that our approach scales well to high-dimensional hardware systems.
We study multi-query collision-free motion planning for the KUKA iiwa robotic arm (\cref{f-robot-arm-title}), tasked with moving virtual items between shelves and bins.
Our methodology requires minimal additional offline computation, while delivering significant online speed up with negligible solution quality reduction.

We first produce an approximate polytopic decomposition of the 7-dimen\-sional collision-free configuration space of the arm.
This is done via the IRIS-NP algorithm \cite{petersen2023growing}, and we use IRIS clique seeding \cite{werner2023approximating} to obtain polytopes inside the shelves and bins.
We assign a GCS vertex $v$ per polytope in this decomposition.
The convex set $\mathcal X_v$ is the set of linear segments contained within the region, with the segment represented by its endpoints.
Two GCS vertices are connected by an edge if the corresponding regions overlap.
The resulting graph contains 23 vertices and 68 edges.
For each edge $e\!=\!(v,w)$, we constrain the linear segments at $v$ and $w$ to form a continuous path.
The path length is the sum of the Euclidean distances of the linear segments.
We define 12 source vertices (6 shelves, 2 vertices per shelf) and 3 target vertices (inside the left, front, and right bins).
To generate the quadratic lower bounds on the cost-to-go function, we use the generalization of \eqref{e-path-cost-to-go-synthesis} discussed in \cref{a-further-generalization}.

\begin{figure}[t!]
    \centering
    \includegraphics[width=\textwidth]{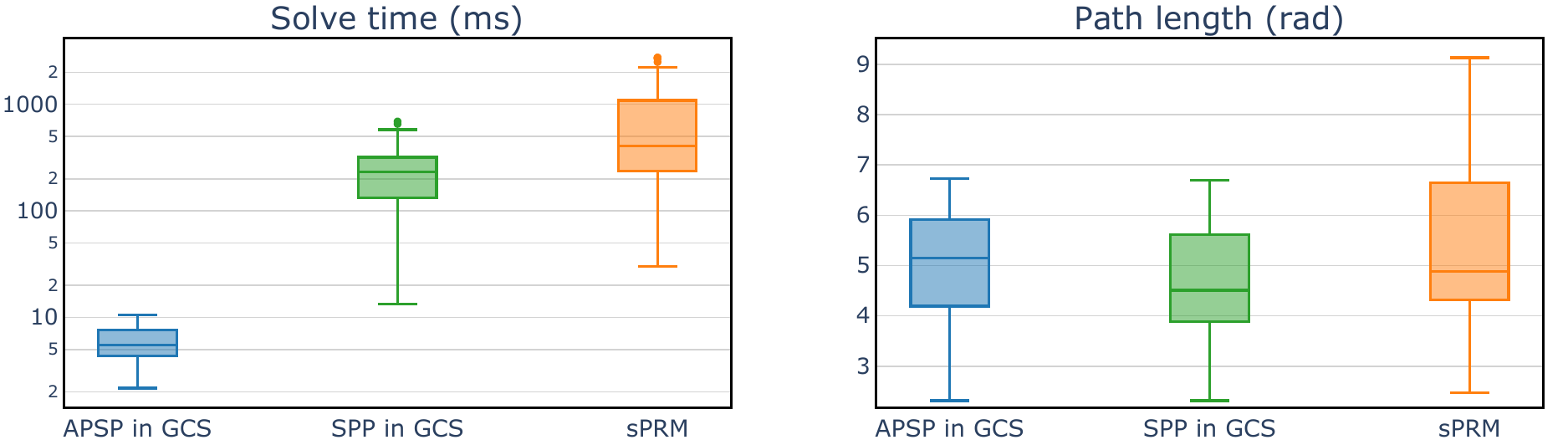}
        \caption{
        For the robot arm scenario in \cref{ss-robot-arm-example}, we compare path length and solve time performance between the APSP in GCS, single-query SPP in GCS, and shortcut PRM over 120 queries.
        The offline phases take 106s, 100s, and 0.9s respectively.
        The APSP in GCS is on average 40 times faster than the SPP in GCS, with minimal reduction in solution quality.
        Compared to sPRM, the APSP in GCS is on average 110 times faster.
    }
    \label{f-arm-numerics}
\end{figure}

We evaluate our algorithm in a multi-query scenario: at runtime, the arm is given a random next position to go to, alternating between shelves and bins.
We rollout a 1-step lookahead policy to generate paths from shelves to bins, and reverse them to obtain paths from bins to shelves.
We evaluate our approach on a total of 120 queries.
We compare our algorithm against solving the SPP in GCS from scratch, as well as against the shortcut PRM (sPRM) algorithm, which is its natural sampling based multi-query competitor.
We use a high-performance implementation of sPRM based on~\cite{prm-rob}, producing a large roadmap with 10,000 vertices.
Our solutions are visualized in \cref{f-robot-arm-title};
performance comparison is provided in \cref{f-arm-numerics}.
Similar to how the quality of the PRM solutions depends on the density of the PRM, the quality of solutions obtained with GCS depends on the quality of the polytopic decomposition of the collision-free configuration space. 
We thus make no claims about the optimality of the solutions in this section.

Offline, generating cost-to-go lower bounds takes only 6 seconds, which is just 6\% of the time that it takes to generate the polytopic decomposition  necessary to use GCS.
Then online, our policy rollouts are very fast, with a median solve time of 5ms and a maximum of 11ms (we report the parallelized solver time).
Our method is on average 40 time faster than the SPP in GCS, producing paths that are only 7\% longer on average.
Compared to sPRM, our method is on average 110 times faster and produces paths that are 5\% shorter on average.
We achieve consistent performance in both solve time and path length, unlike sPRM, which shows high variance in both.
Overall, compared to these state-of-the-art baselines, the APSP in GCS reduces the online solve times significantly, with minimal compromise in solution quality.

\subsection{Scalability and ablation on lower bounds and lookahead horizon}
\label{ss-ablation-example}
In this section, we demonstrate the scalability of our approach and analyze how the coarseness of the cost-to-go lower bounds and the lookahead horizon impact solution quality. 
First, we show that multi-step lookaheads with quadratic $J_{v}$ yield near-optimal solutions in large graphs. 
Second, we demonstrate that quadratic lower bounds significantly outperform the affine ones, which are available from the dual of the convex relaxation of the SPP in GCS~\cite[App. B]{marcucci2024shortest}.

\begin{figure}[t!]
    \centering
    \begin{subfigure}[t]{0.33\textwidth} 
        \centering        
        \includegraphics[width=\textwidth]{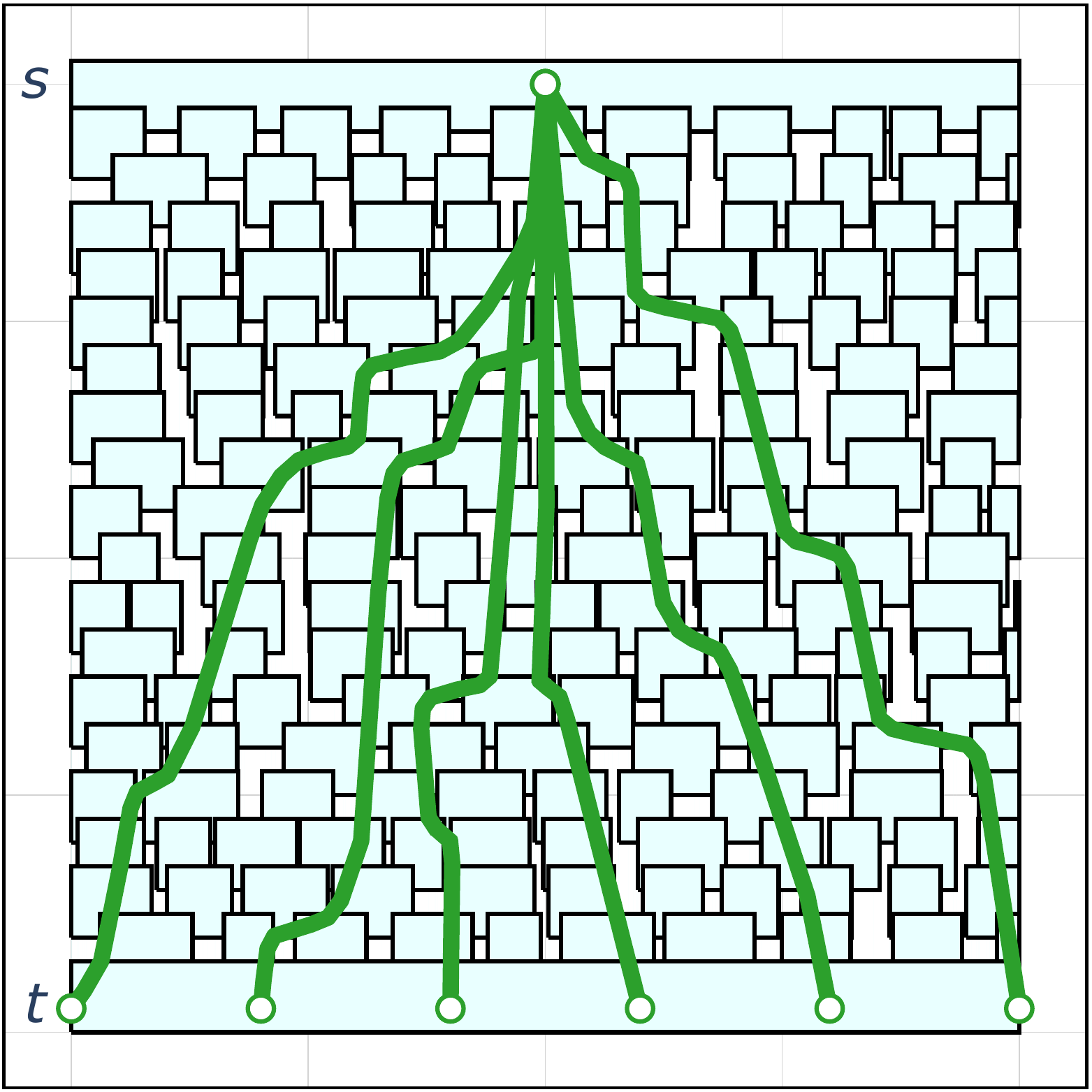}
        \caption{Optimal solution.}
        \label{sf-sa}
    \end{subfigure}\hfill
    \begin{subfigure}[t]{0.33\textwidth} 
        \centering        
        \includegraphics[width=\textwidth]{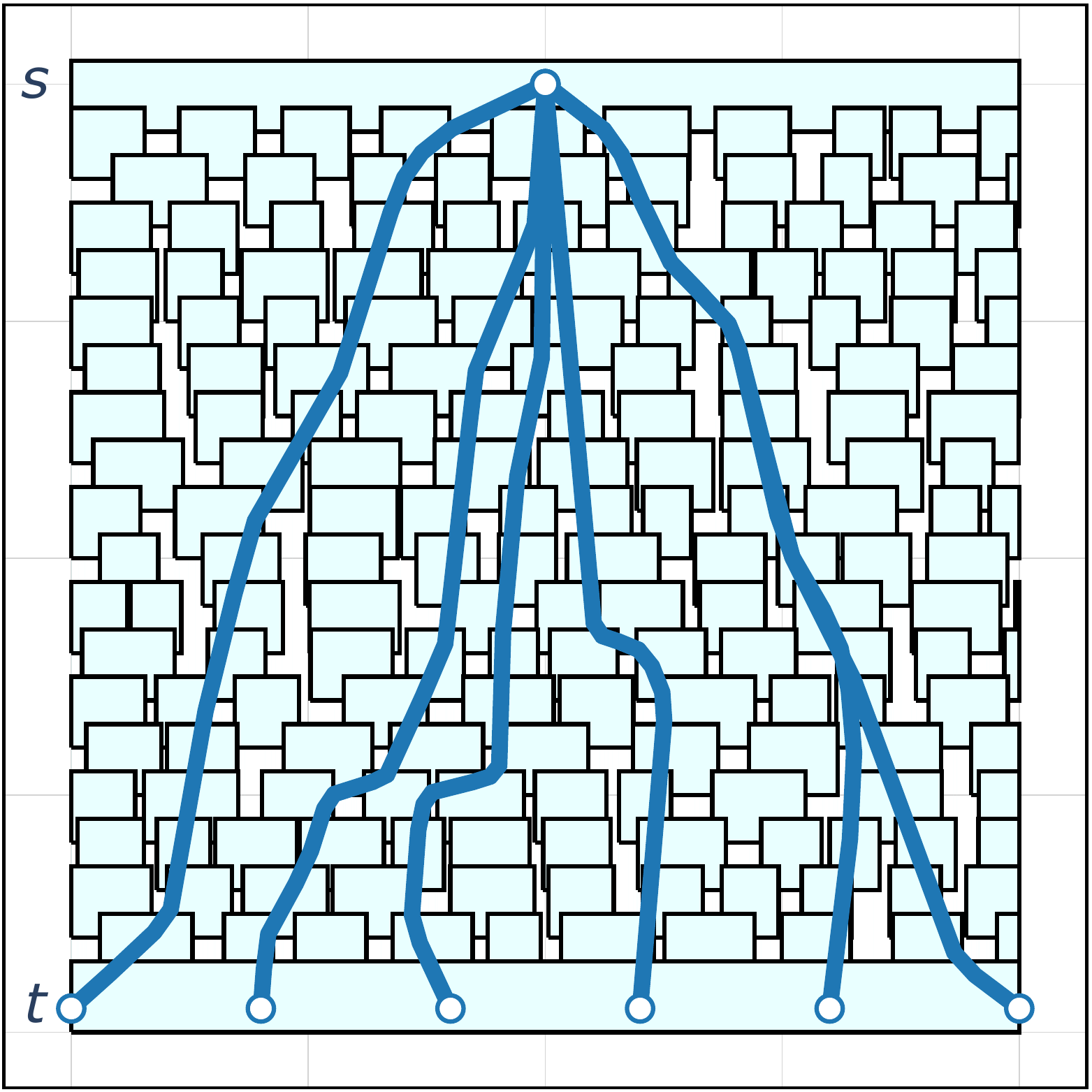}
        \caption{Quadratic lower bounds.}
        \label{sf-sb}
    \end{subfigure}\hfill
    \begin{subfigure}[t]{0.33\textwidth}
        \centering
        \includegraphics[width=\textwidth]{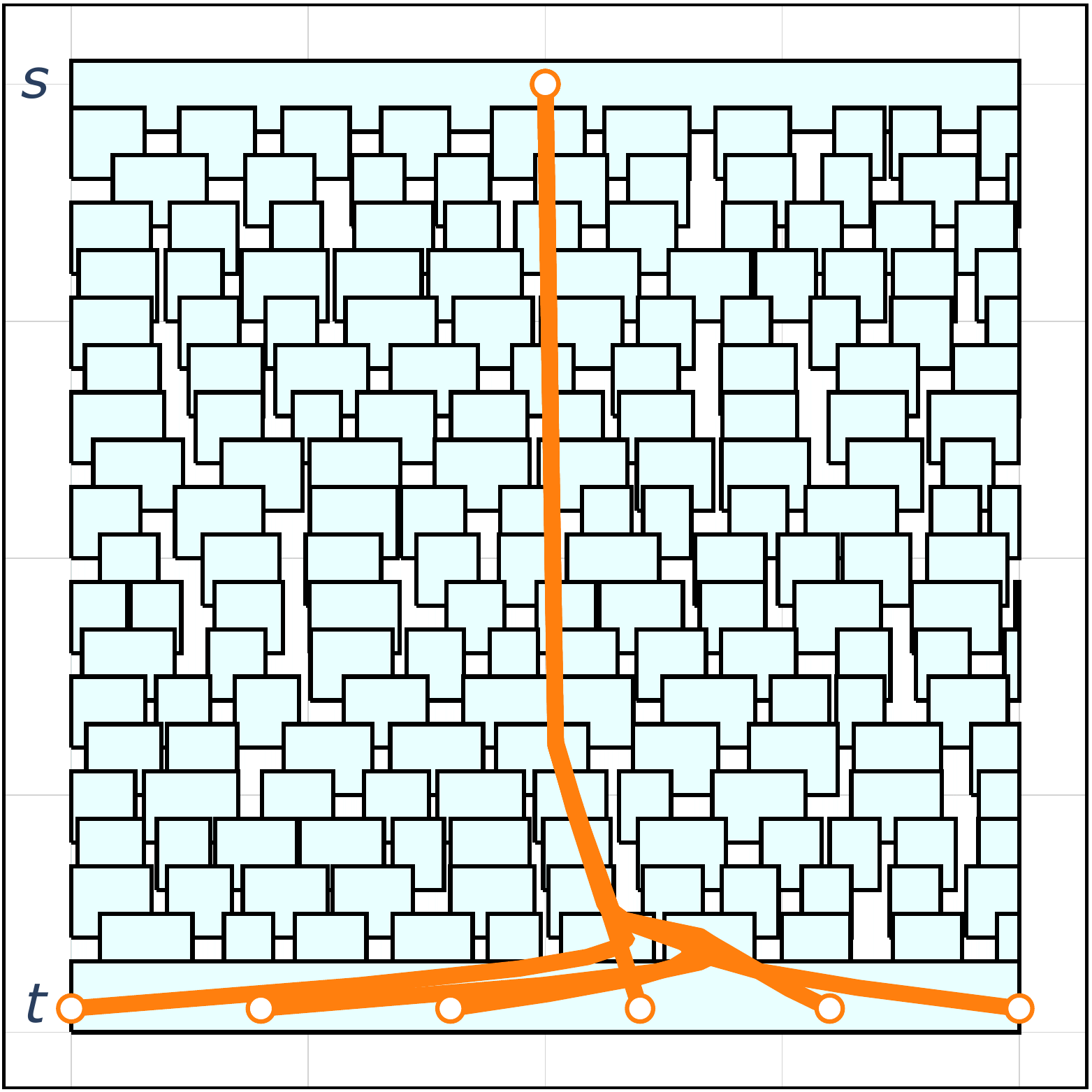}
        \caption{Affine lower bounds.}
        \label{sf-sc}
    \end{subfigure}
    \caption{
    A 3-step lookahead policy with quadratic $J_{v}$ (blue) yields diverse vertex paths resembling the optimal solutions (green).
    A 3-step lookahead with affine $J_{v}$ (orange) follows a single vertex sequence regardless of the target point, accruing much higher cost.
    }
    \label{f-3-birchtree-plots}
\end{figure}

We consider a randomly generated environment depicted in \cref{f-3-birchtree-plots}. 
We assign a GCS vertex $v$ for each teal box.
Each convex set $\mathcal X_v$ is the set of control points of a cubic Bézier curve within the box (see~\cite{marcucci2023motion}).
The GCS vertices are connected by a pair of edges if the corresponding teal boxes overlap.
The resulting graph has 190 vertices and 540 edges.
For each edge, we constrain the vertex Bézier curves to be differentiable at the transition point.
The path cost is the sum of squared Euclidean distances between the consecutive control points of the Bézier curves.
The source vertex $s$ is at the top, and the target vertex $t$ is at the bottom.

We synthesize the quadratic and affine lower bounds over the GCS, which takes 6s and 2s respectively.
We then uniformly sample 120 pairs of source and target conditions, and rollout the greedy policy using different lower bounds and lookahead horizons.
Optimal solutions are obtained by solving the MICP formulation of the SPP in GCS.
Numerical results are reported in \cref{t-numerics-ablation}.

\cref{t-numerics-ablation} shows that our approach scales well to large problem instances, yielding better solve times than the SPP in GCS.
A 2-3 step lookahead policy with a quadratic cost-to-go lower bound produces near-optimal solutions (8-9\% median suboptimality) in under 10ms.
The SPP in GCS produces slightly better solutions (7\% median suboptimality), but due to the size of the graph, the solve-time increases to over 1000ms.
For large graph instances, incremental search through the graph via the APSP in GCS achieves competitive solution quality while reducing solve times by up to two-three orders of magnitude.

Finally, \cref{t-numerics-ablation} shows that quadratic lower bounds with short-horizon lookaheads offer a good balance between expressive power and solve times. 
A 3-step lookahead policy with affine lower bounds has a median suboptimality of 80.2\%, compared to 8.8\% with quadratic lower bounds. 
Achieving similar solution quality with affine lower bounds requires a lookahead horizon of 8-9 steps, but the resulting rollouts take significantly more time. 
\cref{f-3-birchtree-plots} shows that 3-step lookahead rollouts with affine lower bounds fail to capture the diversity of optimal solutions. 
Additionally, low-horizon lookahead policies with affine lower bounds often fail to produce solutions within a reasonable number of iterations, as demonstrated by the failure rate statistics.
Overall, we observe that the lookahead policies with quadratic lower bounds perform much better than those with affine ones.

\begin{table}[t!]
	\centering
	\begin{tabular}{@{}l @{\hspace{0.5em}}| @{\hspace{0.5em}} r @{\hspace{0.5em}} r @{\hspace{0.5em}} r@{}} %
	\toprule
	\textbf{Solution method} & \textbf{Optimality gap, \%} & \textbf{Solve time, ms} & \textbf{Failure rate, \%}  \\ 
	\bottomrule
Quadratic $J_{v}$, 1-step         &  20.0 (62.1)  &  3 (3)  &  0.0 \\
Quadratic $J_{v}$, 2-step         &  9.4 (22.3)   &  4 (4)  &  0.0 \\
Quadratic $J_{v}$, 3-step         &  8.8 (15.7)  &  5 (6)  &  0.0 \\
\midrule
Affine $J_{v}$, 1-step         &  157.1 (\text{N/A})  &  2 (657)  &  27.2 \\
Affine $J_{v}$, 2-step         & 142.4 (418.8)  &  3 (914)  &  14.0 \\
Affine $J_{v}$, 3-step         & 80.2 (348.3)  &  5 (808)  &  9.9 \\
Affine $J_{v}$, 8-step         & 11.9 (37.4)  &  169 (1996)  &  3.3 \\
Affine $J_{v}$, 9-step         & 7.0 (26.2)  &  388 (2454)  &  0.0 \\
\midrule
SPP in GCS          &  6.9 (12.0)  &  716 (1051)  &  0.0 \\
\bottomrule	
	\end{tabular}
    \vspace{5pt}
    \caption{Impact of the degree of $J_{v}$ and lookahead horizon on performance, over 120 queries for the GCS in \cref{f-3-birchtree-plots}.
    We report optimality gaps (ratio between solution cost and optimal cost), solve times, and failure rates (rollout policy is terminated after 10,000 iterations).
    We report median values, with the 75th percentile in the parenthesis.
    Low-horizon lookahead policies with quadratic lower bounds yield near optimal solutions, perform much better than the affine bounds.
    } \label{t-numerics-ablation}
\end{table}

\section{Conclusion and future work}
\label{s-discussion}
In this work, we generalized the classical All-Pairs Shortest-Paths problem to the Graphs of Convex Sets, and developed practical approximate numerical methods for solving this problem.
We demonstrated that a coarse lower bound on the cost-to-go with a greedy multi-step lookahead policy produce near-optimal paths, while significantly reducing solve times.
Our methodology effectively scales to high-dimensional set scenarios and large graph instances, enabling practical robotics applications in multi-query settings.
We plan to provide an efficient implementation of our approach  within the Drake library~\cite{drake}.

For hardware applications in non-static environments, we are interested in ways to tackle changes to the robot's configuration space, like those arising in object manipulation, as well as addition and removal of obstacles.
Assuming the changes are minor, the online search via the multi-step lookahead policy provides natural local adaptation.
Changes to the environment can be incorporated into the online policy rollout program~\eqref{e-policy} via non-convex constraints, similar to~\cite{vonusing}.

Finally, we are interested in exploring alternative incremental search policies beyond the multi-step lookahead policy.
We expect randomized rollouts inspired by MCTS~\cite{browne2012survey} and A*-based approaches like~\cite{chia2024gcs} to be effective.

\clearpage

\bibliographystyle{splncs04}
\bibliography{refs.bib}

\begin{thebibliography}{10}
\providecommand{\url}[1]{\texttt{#1}}
\providecommand{\urlprefix}{URL }
\providecommand{\doi}[1]{https://doi.org/#1}

\bibitem{mosek}
ApS, M.: The MOSEK optimization toolbox for MATLAB manual. Version 10.1.
  (2024), \url{http://docs.mosek.com/latest/toolbox/index.html}

\bibitem{bellman1966dynamic}
Bellman, R.: Dynamic programming. science  \textbf{153}(3731),  34--37 (1966)

\bibitem{bemporad2002explicit}
Bemporad, A., Morari, M., Dua, V., Pistikopoulos, E.N.: The explicit linear
  quadratic regulator for constrained systems. Automatica  \textbf{38}(1),
  3--20 (2002)

\bibitem{bertsekas2012dynamic}
Bertsekas, D.: Dynamic programming and optimal control, vol.~4. Athena
  scientific (2012)

\bibitem{blekherman2012semidefinite}
Blekherman, G., Parrilo, P.A., Thomas, R.R.: Semidefinite optimization and
  convex algebraic geometry. SIAM (2012)

\bibitem{bohlin2000path}
Bohlin, R., Kavraki, L.E.: Path planning using lazy {PRM}. In: Proceedings 2000
  ICRA. Millennium conference. IEEE international conference on robotics and
  automation. Symposia proceedings (Cat. No. 00CH37065). vol.~1, pp. 521--528.
  IEEE (2000)

\bibitem{browne2012survey}
Browne, C.B., Powley, E., Whitehouse, D., Lucas, S.M., Cowling, P.I.,
  Rohlfshagen, P., Tavener, S., Perez, D., Samothrakis, S., Colton, S.: A
  survey of {M}onte {C}arlo tree search methods. IEEE Transactions on
  Computational Intelligence and AI in games  \textbf{4}(1),  1--43 (2012)

\bibitem{chia2024gcs}
Chia, S.Y.C., Jiang, R.H., Graesdal, B.P., Kaelbling, L.P., Tedrake, R.:
  {GCS}*: Forward heuristic search on implicit graphs of convex sets. arXiv
  preprint arXiv:2407.08848  (2024)

\bibitem{cohn2023non}
Cohn, T., Petersen, M., Simchowitz, M., Tedrake, R.: Non-{E}uclidean motion
  planning with graphs of geodesically-convex sets. Robotics: Science and
  Systems  (2023)

\bibitem{cormen2022introduction}
Cormen, T.H., Leiserson, C.E., Rivest, R.L., Stein, C.: Introduction to
  algorithms. MIT press (2022)

\bibitem{de2003linear}
De~Farias, D.P., Van~Roy, B.: The linear programming approach to approximate
  dynamic programming. Operations research  \textbf{51}(6),  850--865 (2003)

\bibitem{floyd1962algorithm}
Floyd, R.W.: Algorithm 97: shortest path. Communications of the ACM
  \textbf{5}(6),  345--345 (1962)

\bibitem{graesdal2024towards}
Graesdal, B.P., Chia, S.Y., Marcucci, T., Morozov, S., Amice, A., Parrilo,
  P.A., Tedrake, R.: Towards tight convex relaxations for contact-rich
  manipulation. Robotics: Science and Systems  (2024)

\bibitem{jaillet2004prm}
Jaillet, L., Sim{\'e}on, T.: A {PRM}-based motion planner for dynamically
  changing environments. In: 2004 IEEE/RSJ International Conference on
  Intelligent Robots and Systems (IROS)(IEEE Cat. No. 04CH37566). vol.~2, pp.
  1606--1611. IEEE (2004)

\bibitem{johnson1977efficient}
Johnson, D.B.: Efficient algorithms for shortest paths in sparse networks.
  Journal of the ACM (JACM)  \textbf{24}(1),  1--13 (1977)

\bibitem{karaman2011sampling}
Karaman, S., Frazzoli, E.: Sampling-based algorithms for optimal motion
  planning. The international journal of robotics research  \textbf{30}(7),
  846--894 (2011)

\bibitem{kavraki1996probabilistic}
Kavraki, L.E., Svestka, P., Latombe, J.C., Overmars, M.H.: Probabilistic
  roadmaps for path planning in high-dimensional configuration spaces. IEEE
  transactions on Robotics and Automation  \textbf{12}(4),  566--580 (1996)

\bibitem{kuffner2000rrt}
Kuffner, J.J., LaValle, S.M.: {RRT}-connect: An efficient approach to
  single-query path planning. In: Proceedings 2000 ICRA. Millennium Conference.
  IEEE International Conference on Robotics and Automation. Symposia
  Proceedings (Cat. No. 00CH37065). vol.~2, pp. 995--1001. IEEE (2000)

\bibitem{kurtz2023temporal}
Kurtz, V., Lin, H.: Temporal logic motion planning with convex optimization via
  graphs of convex sets. IEEE Transactions on Robotics  \textbf{39}(5),
  3791--3804 (2023)

\bibitem{lasserre2001global}
Lasserre, J.B.: Global optimization with polynomials and the problem of
  moments. SIAM Journal on optimization  \textbf{11}(3),  796--817 (2001)

\bibitem{lasserre2008nonlinear}
Lasserre, J.B., Henrion, D., Prieur, C., Tr{\'e}lat, E.: Nonlinear optimal
  control via occupation measures and {LMI}-relaxations. SIAM journal on
  control and optimization  \textbf{47}(4) (2008)

\bibitem{lavalle1998rapidly}
LaValle, S.: Rapidly-exploring random trees: A new tool for path planning.
  Research Report 9811  (1998)

\bibitem{lewis2013reinforcement}
Lewis, F.L., Liu, D.: Reinforcement learning and approximate dynamic
  programming for feedback control. John Wiley \& Sons (2013)

\bibitem{marcucci2024graphs}
Marcucci, T.: Graphs of Convex Sets with Applications to Optimal Control and
  Motion Planning. Ph.D. thesis, Massachusetts Institute of Technology (2024)

\bibitem{marcucci2023motion}
Marcucci, T., Petersen, M., von Wrangel, D., Tedrake, R.: Motion planning
  around obstacles with convex optimization. Science robotics  \textbf{8}(84),
  eadf7843 (2023)

\bibitem{marcucci2024shortest}
Marcucci, T., Umenberger, J., Parrilo, P., Tedrake, R.: Shortest paths in
  graphs of convex sets. SIAM Journal on Optimization  \textbf{34}(1),
  507--532 (2024)

\bibitem{parrilo2000structured}
Parrilo, P.A.: Structured semidefinite programs and semialgebraic geometry
  methods in robustness and optimization. California Institute of Technology
  (2000)

\bibitem{parrilo2003semidefinite}
Parrilo, P.A.: Semidefinite programming relaxations for semialgebraic problems.
  Mathematical programming  \textbf{96},  293--320 (2003)

\bibitem{petersen2023growing}
Petersen, M., Tedrake, R.: Growing convex collision-free regions in
  configuration space using nonlinear programming. arXiv preprint
  arXiv:2303.14737  (2023)

\bibitem{philip2024mixed}
Philip, A.G., Ren, Z., Rathinam, S., Choset, H.: A mixed-integer conic program
  for the moving-target traveling salesman problem based on a graph of convex
  sets. arXiv preprint arXiv:2403.04917  (2024)

\bibitem{prm-rob}
Phillips-Grafflin, C.: Common robotics utilities
  \url{https://github.com/ToyotaResearchInstitute/common_robotics_utilities}

\bibitem{powell2007approximate}
Powell, W.B.: Approximate Dynamic Programming: Solving the curses of
  dimensionality, vol.~703. John Wiley \& Sons (2007)

\bibitem{drake}
Tedrake, R., the Drake Development~Team: Drake: Model-based design and
  verification for robotics (2019), \url{https://drake.mit.edu}

\bibitem{wang2015approximate}
Wang, Y., O'Donoghue, B., Boyd, S.: Approximate dynamic programming via
  iterated {B}ellman inequalities. International Journal of Robust and
  Nonlinear Control  \textbf{25}(10) (2015)

\bibitem{warshall1962theorem}
Warshall, S.: A theorem on boolean matrices. Journal of the ACM (JACM)
  \textbf{9}(1),  11--12 (1962)

\bibitem{werner2023approximating}
Werner, P., Amice, A., Marcucci, T., Rus, D., Tedrake, R.: Approximating robot
  configuration spaces with few convex sets using clique covers of visibility
  graphs. International Conference on Robotics and Automation  (2024)

\bibitem{vonusing}
von Wrangel, D., Tedrake, R.: Using graphs of convex sets to guide nonconvex
  trajectory optimization

\end{thebibliography}

\newpage
\appendix

\section{Extensions and variations}
\label{a-further-generalization}

We briefly remark on various natural generalizations to program \eqref{e-path-cost-to-go-synthesis}.

\begin{enumerate}
\item Suppose the set of source vertices $\mathcal S$ has more than one vertex.
To simultaneously ``push up'' lower bounds $J_s$ per vertex $s\in\mathcal S$, we add extra integral terms to the objective function \eqref{e-p-objective}.
\item Suppose the target set $\mathcal X_t$ is not a singleton, but a compact convex set.
First, we modify the constraint \eqref{e-p-value-def} to search for $J_{v,t}:\mathcal X_v \times \mathcal X_t \rightarrow \mathbb R$.
The function $J_{v,t}(x_v,x_t)$ is a lower bound on the cost-to-go of the shortest path from $x_v$ of vertex $v$ to $x_t$ of vertex $t$.
Similarly, the probability distribution $\phi_{s,t}$ is now supported on $\mathcal X_s\times\mathcal X_t$, so as to push
up on $J_{s,t}(x_s,x_t)$ over all source-target pairs $(x_s,x_t)$.
The lower-bound constraint \eqref{e-p-lower-bound} is adjusted to include $x_t\in\mathcal X_t$:
$$
J_{v,t}(x_v, x_t) \leq  l_e(x_v, x_w) +  h_w + J_{w,t}(x_w, x_t),
$$
for all edges $e=(v,w)\in\mathcal E$, and all points $(x_v,x_w)\in \mathcal X_e$ and $x_t\in\mathcal X_t$.
Finally, the target constraint \eqref{e-p-target-value} is adjusted to be $J_{t,t}(x_t, x_t) = -\sum_{w\in\mathcal V} h_w$, for all $x_t\in\mathcal X_t$.
\item The scalar vertex penalty $h_w$ is generalized to be a non-negative function of the target state $x_t$, that is: $h_{w,t}:\mathcal X_t\rightarrow \mathbb R_+$.
We thus replace $h_w$ with $h_{w,t}(x_t)$ and update the constraint \eqref{e-p-target-value} as follows:
$$
J_{t,t}(x_t, x_t) = -\sum_{w\in\mathcal V} h_w(x_t),
$$
further tightening the resulting lower-bounds.
\item Suppose the set of target vertices $\mathcal T$ has more than one vertex.
To obtain the cost-to-go lower bounds for every pair of vertices $v\in\mathcal V$ and $t\in\mathcal T$, we solve multiple programs \eqref{e-path-cost-to-go-synthesis} in parallel, one per target vertex $t\in\mathcal T$.
\item In general, the successor policy \eqref{e-policy} is also a function of terminal vertex $t$ and terminal point $x_t$. 
The generalized 1-step lookahead program is as follows:
\begin{align*}
\pi(v_k ,x_k, p_k, t, x_t) \quad=\quad \underset{(w, x_w)}{\arg\min}& \quad  l_e(x_k, x_w) + J_{w,t}(x_w, x_t) \\
\text{s.t.} & \quad e=(v_k,w)\in\mathcal E_{v_k}^{\text{out}}, \quad w\notin p_k, \nonumber\\
& \quad (x_k,x_w)\in\mathcal X_e. \nonumber
\end{align*}
\item 
Other penalties, similar to the vertex visitation penalties $h_{v}$, can be added to improve the quality of the lower bounds. 
For instance, consider a 2-cycle with edges $(v,w)$ and $(w,v)$. 
We can add edge penalties $h_{v,w} = h_{w,v}$ for traversing either edge.
By subtracting $h_{v,w}$ from the cost-to-go lower bound at the target, we effectively ensure that no penalty is incurred for traversing just one (but not both) of the edges.
This can be extended to cycles of arbitrary length.
\end{enumerate}

\end{document}